%% file: main.tex
\documentclass{article} 
\usepackage{iclr2027_conference,times}

\input{math_commands.tex}

\usepackage{wrapfig}
\usepackage{booktabs}
\usepackage{xcolor}
\usepackage{colortbl} 
\usepackage[most]{tcolorbox}
\usepackage{multirow}
\usepackage{subcaption}
\usepackage{lmodern}
\usepackage{fancyvrb}
\usepackage{fvextra}
\usepackage{titlesec}
\usepackage{parskip}
\tcbuselibrary{listings,theorems,breakable}
\newtcbtheorem[number within=section]{exmp}{Prompts}%
{breakable,colback=white!5!white,colframe=black!95!,fonttitle=\bfseries, left=.02in, right=.02in,bottom=.02in, top=.02in}{exmp}

\newtcbtheorem[number within=section]{case}{Examples}%
{breakable,colback=white!5!white,colframe=black!95!,fonttitle=\bfseries, left=.02in, right=.02in,bottom=.02in, top=.02in}{case}

\usepackage[utf8]{inputenc} 
\usepackage[T1]{fontenc}    
\usepackage{hyperref}       
\usepackage{url}            
\usepackage{booktabs}       

\usepackage{siunitx}
\usepackage{nicefrac}       
\usepackage{microtype}      
\usepackage{xcolor}         
\definecolor{purple1}{HTML}{e5e0f0}

\usepackage[utf8]{inputenc} 
\usepackage[T1]{fontenc}    
\usepackage{amsfonts}       
\usepackage{hyperref}       
\usepackage{url}            
\usepackage{booktabs}       
\usepackage{amsfonts}       
\usepackage{nicefrac}       
\usepackage{microtype}      
\usepackage{xcolor}         
\usepackage{amsmath}
\usepackage{amssymb}
\usepackage{mathtools}
\usepackage{amsthm}
\usepackage{caption}
\usepackage{graphicx}
\usepackage{multirow}
\usepackage{enumitem}
\usepackage{color}
\usepackage{xcolor}
\usepackage{colortbl}
\usepackage{pifont}
\usepackage{algorithm}
\usepackage{algorithmic}
\usepackage{ulem}
\usepackage{booktabs}
\usepackage{bbding}
\usepackage{wrapfig}
\definecolor{DarkGreen}{RGB}{1,100,32} 
\usepackage[capitalize,noabbrev]{cleveref}

\theoremstyle{plain}
\newtheorem{theorem}{Theorem}[section]

\theoremstyle{definition}

\theoremstyle{remark}

\definecolor{titleblue}{HTML}{1F4E79}
\definecolor{softbg}{HTML}{F7FAFC}
\definecolor{lineblue}{HTML}{90CAF9}
\definecolor{taggray}{HTML}{4A5568}

\newtcolorbox{promptbox}[1][]{
  enhanced,
  breakable,
  colback=softbg,
  colframe=lineblue,
  boxrule=0.6pt,
  arc=2mm,
  left=2mm,
  right=2mm,
  top=1mm,
  bottom=1mm,
  #1
}

\DefineVerbatimEnvironment{PromptVerbatim}{Verbatim}{
  fontsize=\small,
  breaklines=true,
  breakanywhere=true
}

\title{PhyMo: A Physical-Field Modality for Multimodal AI4Physics}

\author{
\textbf{Henan Sun\textsuperscript{1}},
 \textbf{Haitao Hu\textsuperscript{1}},
 \textbf{Jin Liu\textsuperscript{3}},
 \textbf{Jianfeng Zhang\textsuperscript{3}},
 \textbf{Lujia Pan\textsuperscript{3}},
 \textbf{Nuo Chen\textsuperscript{1,4}},
 \textbf{Jia Li \textsuperscript{1,2,*}},
\\
 \textsuperscript{1}The Hong Kong University of Science and Technology (Guangzhou),\\
 \textsuperscript{2}The Hong Kong University of Science and Technology,\\
 \textsuperscript{3}Huawei Noah's Ark Lab,\\
 \textsuperscript{4}Tencent HY
}

\begin{document}

\iclrfinalcopy
\maketitle
\fancyhead{}

\begin{abstract}
Multimodal learning is emerging as a powerful paradigm for AI for Physics (AI4Physics), where predicting physical systems requires the joint interpretation of heterogeneous observations, measurements, and domain knowledge. However, existing approaches typically represent physical quantities and governing equations as generic numerical or textual tokens, overlooking the physical constraints that determine their spatiotemporal interactions. To address this limitation, we introduce the \textbf{physical-field modality} and propose \textbf{PhyMo}, a physics-grounded multimodal framework that organizes heterogeneous measurements through PDE-associated operators. PhyMo follows a three-stage learning procedure: the physical-field encoder is first pretrained through field reconstruction under PDE residual supervision, its representations are subsequently aligned with visual embeddings in a shared latent space, and the fused multimodal representations are finally processed by corresponding downstream prediction heads. Experiments on five datasets spanning diverse physical environments show that PhyMo achieves state-of-the-art performance, compared to the strongest baseline on each dataset, demonstrating the superiority of PhyMo on multimodal representation learning in AI4Physics.

\end{abstract}

\section{Introduction}
\label{sec: intro}

Artificial intelligence (AI), and multimodal learning in particular, has emerged as a powerful paradigm for addressing complex scientific and societal challenges that require the integration of heterogeneous information sources~(\cite{yang2026unified}). This capability is especially valuable for AI for Physics (AI4Physics), where understanding and predicting physical-world systems often requires jointly modeling multimodal observations, measurements, and domain-specific physical knowledge~(\cite{abbas2026physics,yang2026unified}).
Recent advances in multimodal foundation models have demonstrated the potential of combining heterogeneous modalities, such as images, text, and numerical measurements, to address increasingly complex AI4Physics problems~(\cite{huang2025multimodal,tang2026multimodal,zou2026intern}). A common paradigm is to encode different sources of scientific information into a shared representation space, enabling a unified model to reason over observational data and domain knowledge. Thus, visual observations can provide rich spatial information, while textual representations can conveniently encode physical laws, mathematical formulas, scientific parameters, and numerical measurements~(\cite{zhu2026foundations,negrini2025multimodal}). Consequently, recent AI4Physics systems implicitly treat physical equations, parameters, and scientific data as textual or tokenized sequences, thereby enabling existing multimodal architectures to process heterogeneous scientific information.

However, directly representing physical information as text or generic tokens overlooks an essential property of physical data: physical quantities are not merely information symbols, but variables governed by physical laws. Treating them as ordinary text tokens may preserve their semantic identity while discarding the mathematical constraints that determine how they interact~(\cite{shen2025position}). Taking the scenario of solar photovoltaic (PV) power forecasting as example, the data of wind speed is not simply an independent numerical feature; its spatial and temporal evolution is associated with transport and advection processes that can be described by advection-type partial differential equations (PDEs). Similarly, the propagation and spatial variation of solar irradiance are closely related to physical transport and diffusion processes, which can be characterized through diffusion-type PDEs. If wind speed, solar irradiance, and other physical measurements are directly serialized as text tokens and fed into a multimodal model, the model is not explicitly informed that these quantities should satisfy the corresponding physical relationships. As a result, the model may learn statistical correlations from the observed data while failing to adequately exploit the governing physical constraints~(\cite{nie2023skipp}).

Motivated by this limitation, we introduce a new modality: \textbf{physical-field modality}, and the corresponding model \textbf{PhyMo}, which explicitly represents physical measurements according to the governing structures of physical systems rather than treating them as generic tokens. The key observation is that a broad class of physical laws can be formulated in terms of partial differential equations, providing a unified mathematical language for describing the evolution and interaction of physical fields. Based on this observation, we organize physical measurements into a set of PDEs that explicitly characterize their temporal and spatial dynamics, and introduce physics-based constraints through PDE residual losses during model training. The resulting physical-field modality enables PhyMo to learn representations that are simultaneously grounded in observational information and constrained by the underlying physical dynamics. Importantly, our formulation separates the representation of physical fields from the visual modality while retaining a unified multimodal learning framework, making the physical encoder naturally extensible to different physical systems and governing equations. This design therefore provides not only a principled mechanism for injecting physical knowledge into multimodal models, but also a scalable pathway toward incorporating diverse physical fields and PDEs into foundation-model-style architectures. By jointly exploiting visual observations and physics-grounded representations, our approach aims to learn predictive models that are both statistically accurate and physically consistent.

\textbf{Our Contributions.}
(1) \textit{\underline{New Physical-Field Modality.}}
To the best of our knowledge, this work is the first to formulate physical fields as a dedicated modality for multimodal AI4Physics, rather than treating physical measurements as generic numerical or textual inputs. By organizing heterogeneous physical quantities through PDE-associated operators, our formulation provides a unified representation that explicitly preserves the spatiotemporal structure and governing relationships of physical systems.
(2) \textit{\underline{Physics-Grounded Multimodal Framework.}}
We propose PhyMo, a three-stage framework consisting of physical-field encoder pretraining, multimodal alignment, and downstream prediction-head training. The physical-field encoder is pretrained through field reconstruction with PDE residual supervision to capture both physical states and governing dynamics. The learned physical embeddings are then aligned with visual representations in a shared latent space, and their fused representations are finally used by prediction heads for downstream prediction.
(3) \textit{\underline{Theoretical and Empirical Effectiveness.}}
We provide rigorous theoretical analyses showing the advantage of complementary multimodal information, physics-constrained representation learning, and PDE-structured physical inputs over conventional numerical or text representations. Extensive experiments further demonstrate that the proposed PhyMo consistently enhances multimodal representation quality and improves downstream forecasting performance, validating its effectiveness as a general representation paradigm for AI4Physics.

\section{Related Work}
\label{sec: related_work}

Artificial intelligence has demonstrated strong potential for modeling complex physical-world systems. 
Pangu-Weather~(\cite{bi2023accurate}) employs a 3D Earth-specific Transformer and hierarchical temporal aggregation to model global atmospheric evolution across multiple pressure levels. 
GraphCast~(\cite{lam2023learning}) represents the global atmosphere as a multi-scale graph and performs autoregressive message passing for medium-range weather forecasting. 
FuXi~(\cite{chen2023fuxi}) adopts a cascaded forecasting architecture with specialized models for different lead-time ranges to mitigate error accumulation in long-range prediction. 
ClimaX~(\cite{nguyen2023climax}) introduces a Transformer-based foundation model with variable-specific tokenization and aggregation for transferable weather and climate modeling. 
NeuralGCM~(\cite{kochkov2024neural}) combines a differentiable atmospheric dynamical core with neural parameterizations of unresolved physical processes for weather and climate simulation. 
GenCast~(\cite{price2025probabilistic}) formulates global probabilistic weather forecasting as conditional diffusion over atmospheric states to generate ensemble trajectories. 
Aurora~(\cite{bodnar2025foundation}) employs Perceiver-based encoders and decoders with a 3D Swin Transformer to learn transferable representations across diverse Earth-system forecasting tasks. 
Despite their strong predictive capability, these approaches predominantly operate on a single structured physical-data representation and do not explicitly exploit complementary modalities, limiting their ability to learn richer cross-modal representations of physical systems.

More recently, multimodal learning has been introduced to physical-science problems by jointly modeling complementary observations and scientific knowledge.
AstroCLIP~(\cite{parker2024astroclip}) independently pretrains Transformer-based encoders for galaxy images and optical spectra, and subsequently aligns the two modalities through contrastive learning to obtain a shared representation for redshift estimation, physical-property prediction, and morphology classification.
Maven~(\cite{zhang2024maven}) uses modality-specific Transformer encoders and contrastive learning to align supernova photometry and spectroscopy within a shared representation space. 
MultiMat~(\cite{moro2025multimodal}) jointly aligns crystal structures, density of states, charge-density fields, and textual descriptions through self-supervised multimodal contrastive learning for material-property prediction and discovery. 
AION-1~(\cite{parker2026aion}) unifies heterogeneous astronomical observations, including multiband images, spectra, and scalar measurements, through modality-specific tokenization followed by Transformer-based multimodal masked modeling for a broad range of downstream astrophysical tasks.
RadarQA~(\cite{he2026radarqa}) adapts multimodal large language models to weather-forecast analysis by jointly processing radar observations, forecast sequences, physical attributes, and natural-language assessment reports.
LLaMA-Vision for Neutrino Classification~(\cite{sagar2026adapting}) fine-tunes LLaMA~3.2 Vision with QLoRA to jointly process detector pixel maps and physics-informed textual prompts, enabling multimodal classification and explanation of neutrino interaction events in high-energy physics.
While these methods demonstrate the benefit of multimodal fusion, physical information is generally represented as observational arrays, symbolic or textual descriptions. However, the governing physical field itself is not explicitly formulated as a dedicated modality whose representation learning is constrained by its physical laws. 
Consequently, cross-modal representations are primarily learned from statistical correspondence between modalities rather than explicit governing dynamics, motivating our PDE-structured physical-field modality with physics-constrained representation learning.

\section{Methodology}
\label{sec: methodology}

\subsection{Motivation}
\label{sec: Motivation}
As discussed in Section~\ref{sec: intro}, existing multimodal paradigms for AI4Physics predominantly integrate scientific information through conventional modalities, such as images and text. Although such formulations provide a convenient interface for incorporating heterogeneous data, they often treat physical measurements and equations as generic numerical or textual tokens, without explicitly preserving the physical constraints governing their spatiotemporal evolution. This motivates us to introduce a dedicated \textbf{physical-field modality}. Ideally, such a modality should satisfy two fundamental requirements. First, it should provide a general and convenient representation for heterogeneous physical measurements, allowing different physical quantities to be systematically mapped into a common representation space. Second, the learned embeddings should be explicitly grounded in the governing physical laws, such that the representation captures not only statistical correlations among observations but also the underlying physical dynamics. In other words, we seek a modality-specific representation $\mathbf{z}_{\mathrm{phy}}$ whose encoder \begin{equation} f_{\mathrm{phy}}: \mathcal{X}_{\mathrm{phy}} \rightarrow \mathcal{Z}_{\mathrm{phy}} \end{equation} maps heterogeneous physical observations $\mathcal{X}_{\mathrm{phy}}$ into a latent space $\mathcal{Z}_{\mathrm{phy}}$ while preserving the physical constraints associated with the corresponding fields. 

In machine learning, a modality is not only characterized by the raw form of its data, but by the joint design of its data organization, model architecture, training paradigm, and learning objective~(\cite{liang2024foundations,xu2023multimodal,zong2024self}). For example, language models organize text as discrete tokens, process them with Transformer-based architectures, and learn representations through token-prediction objectives such as next-token prediction. Note that physical measurements are fundamentally governed by physical laws, which can be expressed as partial differential equations (PDEs). Thus, PDEs provide a natural abstraction for organizing heterogeneous physical data. Accordingly, instead of treating physical measurements, such as temperature, radiation, wind and humidity, as isolated scalar inputs, we organize them as corresponding PDEs and enforce physical field reconstruction with PDE residual losses. Reconstruction preserves information about the physical state, while PDE-based supervision imposes physics-aware inductive bias on the learned representation. Under this design, physical quantities and PDE operators define the data structure, the encoder maps them into a latent space, field reconstruction defines the learning paradigm, and PDE residual loss provides the physics-grounded objective. The resulting embeddings therefore capture not only the observed physical data, but also the governing structure of the underlying physical system.

\subsection{The Framework of PhyMo}
\label{sec: framework}
Our framework of PhyMo is designed to integrate visual observations with the proposed physical-field modality for multimodal forecasting. As illustrated in Fig.~\ref{fig:framework}, the overall training procedure consists of three stages: 
\textbf{(1) Physical-field Encoder Pretraining}, 
\textbf{(2) Multimodal Alignment}, and 
\textbf{(3) Prediction Head Training}.
The first stage learns physics-aware representations from structured physical fields, the second stage aligns physical and visual representations into a compatible latent space, and the third stage performs target prediction based on the aligned multimodal embeddings.

\begin{figure*}[t]
    \centering
    \includegraphics[width=\textwidth]{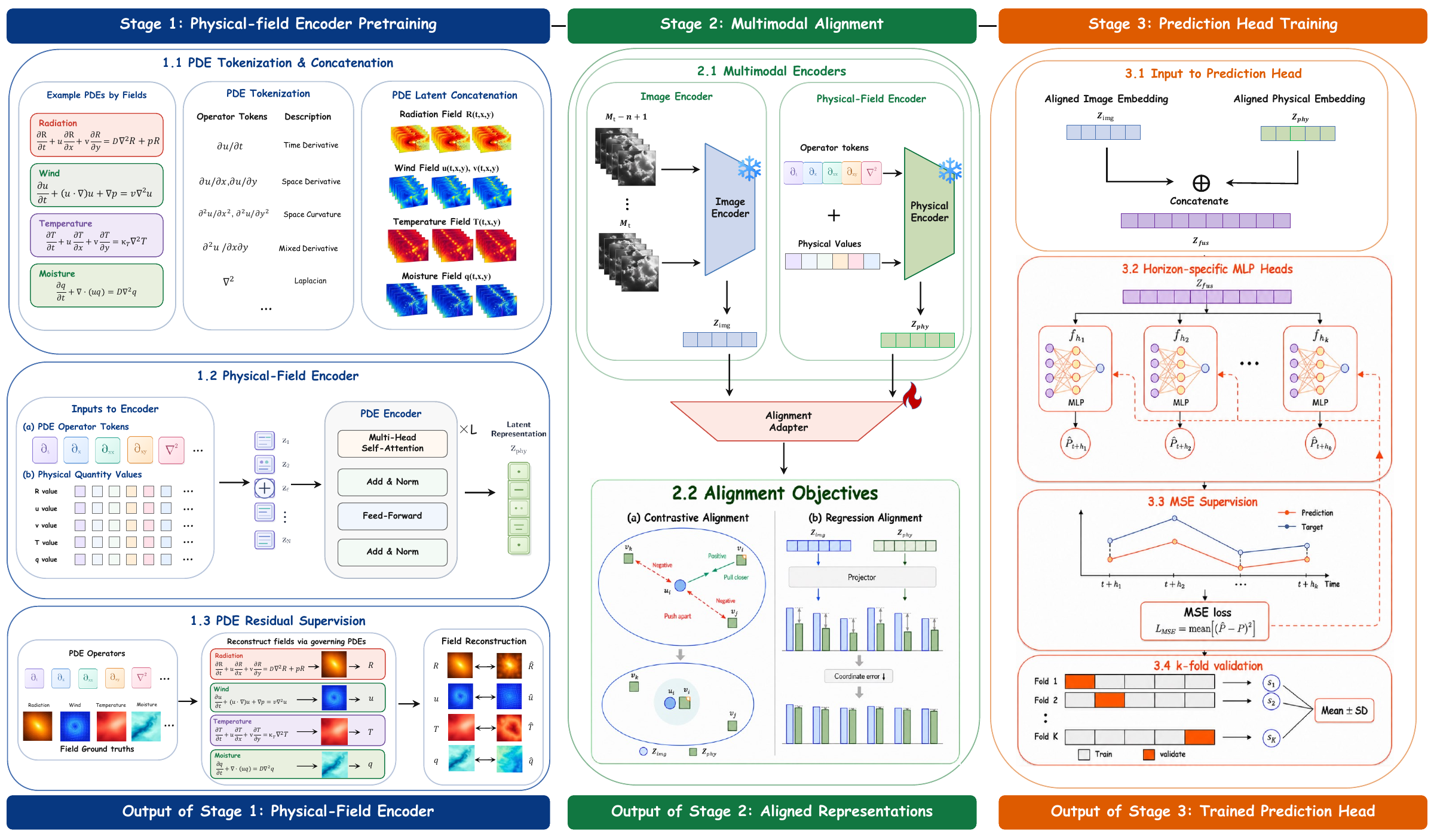}
    \caption{
    Overview of the proposed PhyMo.
    The framework consists of three stages:
    \textbf{Stage 1}, physical-field encoder pretraining with PDE-based tokenization and residual supervision;
    \textbf{Stage 2}, multimodal alignment between visual and physical-field representations;
    and \textbf{Stage 3}, prediction head training for downstream tasks.
    }
    \label{fig:framework}
\end{figure*}

\paragraph{Stage 1: Physical-field Encoder Pretraining.} 
Given multiple physical fields, including radiation, wind, temperature, and moisture, we describe their spatiotemporal dynamics through the corresponding governing PDEs. 
Each physical field is represented by two complementary components: 
\textit{PDE operator tokens}, which characterize the differential structure of the governing equations, and 
\textit{physical quantity values}, which provide the numerical state of the physical system.

Specifically, the PDEs are decomposed into a collection of operators, such as temporal derivatives, spatial derivatives, mixed derivatives, and Laplacian operators. 
These operator tokens are combined with the corresponding physical measurements and fed into a Transformer-style physical-field encoder. 
For a physical input $\mathbf{x}_{\mathrm{phy}}$, the encoder produces a latent representation
\begin{equation}
    \mathbf{z}_{\mathrm{phy}}
    =
    E_{\mathrm{phy}}(\mathbf{x}_{\mathrm{phy}}),
\end{equation}
where $E_{\mathrm{phy}}$ denotes the physical-field encoder.

To ensure that the learned representation preserves the underlying physical structure, we train the encoder through physical-field reconstruction under PDE supervision. 
The latent representation is used to reconstruct the corresponding physical fields, while the reconstructed fields are constrained by their governing PDEs. 
Let $\hat{\mathbf{X}}_{\mathrm{phy}}$ denote the reconstructed fields and $\mathcal{F}_m(\cdot)$ denote the governing PDE for the $m$-th physical field. 
The PDE residual supervision can be generally written as
\begin{equation}
    \mathcal{L}_{\mathrm{PDE}}
    =
    \sum_{m=1}^{M}
    \left\|
        \mathcal{F}_m
        \left(
            \hat{\mathbf{X}}_{\mathrm{phy}}^{(m)}
        \right)
    \right\|_2^2.
\end{equation}

\paragraph{Stage 2: Multimodal Alignment.}
After pretraining the physical-field encoder, we align its representation space with the visual representation space. 
Given a sequence of historical sky images
$
    \mathcal{I}_t
    =
    \left\{
        I_{t-n+1}, \ldots, I_t
    \right\},
$
an image encoder $E_{\mathrm{img}}$ (e.g., pretrained GNNs or ViTs) extracts the corresponding visual embedding
$
    \mathbf{z}_{\mathrm{img}}
    =
    E_{\mathrm{img}}(\mathcal{I}_t).
$
In parallel, the pretrained physical-field encoder produces
$
    \mathbf{z}_{\mathrm{phy}}
    =
    E_{\mathrm{phy}}(\mathbf{x}_{\mathrm{phy}}).
$
Since the two encoders are trained from fundamentally different modalities, their latent spaces are not necessarily directly compatible. 
We therefore introduce an alignment adapter to project the visual and physical representations into a shared latent space. 
Denoting the projected representations as
$\tilde{\mathbf{z}}_{\mathrm{img}}$ and
$\tilde{\mathbf{z}}_{\mathrm{phy}}$, the alignment stage optimizes two complementary objectives.

First, a contrastive objective encourages representations from matched image--physical-field pairs to remain close while separating mismatched pairs:
\begin{equation}
    \mathcal{L}_{\mathrm{con}}
    =
    - \frac{1}{B}
    \sum_{i=1}^{B}
    \log
    \frac{
        \exp
        \left(
            \operatorname{sim}
            (
                \tilde{\mathbf{z}}_{\mathrm{img}}^{i},
                \tilde{\mathbf{z}}_{\mathrm{phy}}^{i}
            )/\tau
        \right)
    }{
        \sum_{j=1}^{B}
        \exp
        \left(
            \operatorname{sim}
            (
                \tilde{\mathbf{z}}_{\mathrm{img}}^{i},
                \tilde{\mathbf{z}}_{\mathrm{phy}}^{j}
            )/\tau
        \right)
    },
\end{equation}
where $B$ denotes the batch size, $\operatorname{sim}(\cdot,\cdot)$ is a similarity function, and $\tau$ is the temperature parameter.
Second, a regression-based alignment objective further reduces the discrepancy between paired representations:
\begin{equation}
    \mathcal{L}_{\mathrm{reg}}
    =
    \frac{1}{B}
    \sum_{i=1}^{B}
    \left\|
        \tilde{\mathbf{z}}_{\mathrm{img}}^{i}
        -
        \tilde{\mathbf{z}}_{\mathrm{phy}}^{i}
    \right\|_2^2.
\end{equation}
The overall multimodal alignment objective is therefore
\begin{equation}
    \mathcal{L}_{\mathrm{align}}
    =
    \mathcal{L}_{\mathrm{con}}
    +
    \lambda_{\mathrm{reg}}
    \mathcal{L}_{\mathrm{reg}},
\end{equation}
where $\lambda_{\mathrm{reg}}$ balances the two objectives.

\paragraph{Stage 3: Prediction Head Training.}
In the final stage, we use the aligned visual and physical representations for downstream tasks. 
The two modality embeddings are concatenated as
$
    \mathbf{z}_{\mathrm{fus}}
    =
    \tilde{\mathbf{z}}_{\mathrm{img}}
    \oplus
    \tilde{\mathbf{z}}_{\mathrm{phy}},
$
where $\oplus$ denotes feature concatenation.
Rather than using a single prediction function for all forecasting horizons, we employ a set of horizon-specific MLP heads
$
    \left\{
        f_{h_1},
        f_{h_2},
        \ldots,
        f_{h_K}
    \right\},
$
where each $f_{h_k}$ is responsible for prediction at a specific future horizon $h_k$. 
The prediction at horizon $h_k$ is given by
$
    \hat{P}_{t+h_k}
    =
    f_{h_k}
    \left(
        \mathbf{z}_{\mathrm{fus}}
    \right).
$
Then, the prediction heads are optimized using mean squared error between the predicted and ground-truth:
\begin{equation}
    \mathcal{L}_{\mathrm{pred}}
    =
    \frac{1}{K}
    \sum_{k=1}^{K}
    \left(
        \hat{P}_{t+h_k}
        -
        P_{t+h_k}
    \right)^2.
\end{equation}

Furthermore, we perform $k$-fold validation during prediction-head training to evaluate model stability and select the prediction model. 
The validation results across different folds are aggregated to obtain the final performance estimate. 
Through this three-stage training procedure, the proposed framework progressively learns physics-informed representations and establishes a semantically consistent cross-modal alignment between physical fields and visual observations.

\subsection{Theoretical Analysis}
\label{sec: theo_analysis}

We provide a theoretical analysis from three complementary perspectives to characterize the role of multimodal fusion, physics-constrained representation learning, and PDE-structured physical information. Detailed assumptions and proofs are deferred to Appendix~\ref{app:theory}.

We first establish that jointly exploiting visual and physical information is theoretically preferable to relying on either modality alone. Let $\mathbf{Z}_{\mathrm{img}}$ and $\mathbf{Z}_{\mathrm{phy}}$ denote the image and physical-field representations, respectively, and let $Y$ denote the target.

\begin{theorem}[Multimodal Bayes-Risk Dominance]
\label{thm:multimodal}
Let
\[
\mathcal{R}^{*}(\mathbf{Z})
=
\inf_f
\mathbb{E}
\left[
\left(
Y-f(\mathbf{Z})
\right)^2
\right]
\]
denote the minimum achievable prediction risk given representation $\mathbf{Z}$. Then
\begin{equation}
\mathcal{R}^{*}
\left(
\mathbf{Z}_{\mathrm{img}},
\mathbf{Z}_{\mathrm{phy}}
\right)
\leq
\min
\left\{
\mathcal{R}^{*}
\left(
\mathbf{Z}_{\mathrm{img}}
\right),
\mathcal{R}^{*}
\left(
\mathbf{Z}_{\mathrm{phy}}
\right)
\right\}.
\label{eq:multimodal_risk}
\end{equation}
\end{theorem}

Theorem~\ref{thm:multimodal} establishes an information-theoretic motivation: incorporating complementary physical information cannot increase the Bayes-optimal risk relative to either image-only or physics-only prediction, and can strictly reduce it when the two modalities contain complementary predictive information.

We next analyze whether the PDE supervision in Stage~1 can improve the quality of the learned physical representation. Let
$\mathbf{X}_{\mathrm{phy}}^{*}$ denote the underlying physical fields and define the physically admissible set as
$
\mathcal{S}
=
\left\{
\mathbf{X}:
\mathcal{F}_m
\left(
\mathbf{X}^{(m)}
\right)=0,
\ \forall m=1,\ldots,M
\right\}.
$
Stage~1 optimizes the reconstruction objective together with
$
\mathcal{L}_{\mathrm{PDE}}
=
\sum_{m=1}^{M}
\left\|
\mathcal{F}_m
\left(
\hat{\mathbf{X}}_{\mathrm{phy}}^{(m)}
\right)
\right\|_2^2.
$

\begin{theorem}[Reconstruction-Risk Advantage of Physics-Constrained Representations]
\label{thm:physics_reconstruction}
Assume that the governing PDEs are correctly specified such that
$\mathbf{X}_{\mathrm{phy}}^{*}\in\mathcal{S}$,
and that $\mathcal{S}$ is a nonempty closed convex set.
For an arbitrary unconstrained reconstruction
$\tilde{\mathbf{X}}_{\mathrm{phy}}$,
let the physics-constrained reconstruction be
$
\hat{\mathbf{X}}_{\mathrm{phy}}^{*}
=
\Pi_{\mathcal{S}}
\left(
\tilde{\mathbf{X}}_{\mathrm{phy}}
\right),
$
where $\Pi_{\mathcal{S}}$ denotes the Euclidean projection onto
$\mathcal{S}$.
Then
\begin{equation}
\mathbb{E}
\left[
\mathcal{L}_{\mathrm{rec}}
\left(
\mathbf{Z}_{\mathrm{phy}}^{*}
\right)
\right]
\leq
\mathbb{E}
\left[
\mathcal{L}_{\mathrm{rec}}
\left(
\mathbf{Z}
\right)
\right],
\label{eq:expected_reconstruction_dominance}
\end{equation}
for the corresponding unconstrained representation $\mathbf{Z}$.
\end{theorem}

Finally, we distinguish the proposed physical-field modality from directly feeding physical measurements as ordinary numerical features.

\begin{theorem}[Theoretical Advantage of the Physical-Field Modality over Numerical-only Inputs]
\label{thm:pde_structure}
Let $\mathbf{X}$ denote the observed physical quantities and let $\mathcal{D}(\mathbf{X})$ collect the PDE-associated information used by the proposed modality in Stage~1. Then
\begin{equation}
\mathcal{R}^{*}
\left(
\mathbf{X},
\mathcal{D}(\mathbf{X})
\right)
\leq
\mathcal{R}^{*}
\left(
\mathbf{X}
\right).
\label{eq:pde_vs_numeric}
\end{equation}
\end{theorem}

Theorem~\ref{thm:pde_structure} provides a theoretical distinction between the proposed physical-field modality and conventional numerical inputs. For example, two physical systems may exhibit identical instantaneous measurements while following different local dynamics. Thus, state-only numerical representations cannot distinguish such cases, whereas PDE-structured representations can encode their distinct evolution. This result therefore motivates organizing physical information according to its governing structure rather than treating it as an unordered collection of scalar features.

\section{Experiments}
\label{sec: experiments}
In this section, we present comprehensive experiments evaluating the proposed PhyMo framework under diverse datasets, including SKIPP'D~(\cite{nie2023skipp}), Folsom~(\cite{pedro2019comprehensive}), NREL~(\cite{hammond2026nrel}), MeteoNet~(\cite{larvor2020meteonet}) and Boreas~(\cite{burnett2023boreas}). The objective of these experiments is to address the following research questions: \textbf{Q1}: How does the proposed PhyMo framework perform compared to other AI4Physics baselines? \textbf{Q2}: To what extent do the proposed components of PhyMo contribute to its multimodal representation capabilities on physical-world datasets? \textbf{Q3}: How sensitive is the performance of the proposed model to variations in its hyperparameters? \textbf{Q4}: What insights can be obtained from the prediction visualization of the proposed model? Due to page limit, we put part of the illustration in experiments in Appendix~\ref{appen: pre-processing}.

\subsection{Experiment Setup}
\label{sec: setup}
\paragraph{Paradigms of Downstream Tasks.}
The proposed PhyMo is designed as a general multimodal representation-learning framework for AI4Physics rather than a task-specific forecasting model. The learned representations can be transferred to different downstream tasks through lightweight task-specific heads. However, due to the page limit, we have to instantiate regression as the primary downstream evaluation paradigm and consider representative tasks in the experiments, such as photovoltaic power and solar irradiance forecasting. Moreover, we conduct 5 experiments over varying random seeds and obtain the mean $R^2$ $\pm$ sample standard deviation, aiming for the statistically-robust evaluation.

\textbf{Datasets.} We evaluate our framework on five physical-world datasets containing images and complementary physical observations across diverse physical scenarios: (1) power prediction: SKIPP'D~(\cite{nie2023skipp}); (2) radiance prediction: Folsom~(\cite{pedro2019comprehensive}) and NREL~(\cite{hammond2026nrel}); (3) weather forecasting: MeteoNet~(\cite{larvor2020meteonet}); (4) autonomous-driving speed prediction: Boreas~(\cite{burnett2023boreas}).
Due to page limit, we have to put the detailed introduction about datasets in Appendix~\ref{appen: datasets}.

\textbf{Baselines.} As illustrated above, we choose regression as the paradigm of downstream tasks. Thus, we compare the performance of the proposed PhyMo against 10 baselines, including: (1) classical models: LASSO~(\cite{ranstam2018lasso}), RF~(\cite{pavlov2000random}) and TabM~(\cite{gorishniy2025tabm}); (2) time-series models: LSTM~(\cite{hochreiter1997long}), TSMixer~(\cite{chen2023tsmixer}) and xLSTM~(\cite{beck2024xlstm}); and (3) the SOTA multimodal models: BiMamba~(\cite{zhang2025multimodal}), AstroCLIP~(\cite{parker2024astroclip}), Maven~(\cite{zhang2024maven}) and AION-1~(\cite{parker2026aion}). The detailed introduction about the selected baselines can be found in Appendix~\ref{appen: baselines}.

\textbf{Evaluation Metrics.} We evaluate forecasting performance using three widely adopted regression metrics: root mean squared error (RMSE), coefficient of determination ($R^2$), and mean absolute error (MAE). RMSE measures the overall prediction deviation while assigning larger penalties to large errors, MAE quantifies the average absolute discrepancy between predictions and ground-truth values, and $R^2$ evaluates the proportion of target variance explained by the model, with higher values indicating better predictive performance. For readability, we report $R^2$ in percentage form, i.e., $100\times R^2$, in all experimental tables. The detailed illustration can be found in Appendix~\ref{appen: metrics}.

\textbf{Environments.} For reproducibility, we report the hardware and software configurations used in our experiments. All experiments were conducted on a server equipped with an Intel(R) Xeon(R) Gold 6240 CPU @ 2.60GHz and a high-performance GPU with 80GB memory.

\subsection{Performance Comparison}
\label{sec: performance_comparison}
Table~\ref{tab:sheet1-r2} shows that PhyMo achieves SOTA performance on SKIPP'D, Folsom, NREL, and MeteoNet, while ranking second on Boreas with performance close to the best-performing method. This consistent competitiveness across solar forecasting, meteorological prediction, and autonomous-driving scenarios demonstrates that PhyMo is not tailored to a specific task or data configuration. Instead, it provides a general framework for integrating visual observations with heterogeneous physical measurements.
The results of Table~\ref{tab:sheet1-r2} illustrate that the advantages of PhyMo are particularly evident on SKIPP'D, NREL and MeteoNet, which contain multiple interacting physical variables and spatially structured observations. This supports our central motivation illustrated in Section~\ref{sec: Motivation}: organizing heterogeneous measurements as the physical-field modality and grounding their representations through field reconstruction and PDE residual supervision can capture the governing dependencies that are difficult to learn. Due to the facts that prior work only treat physical measurements as numerical text and subsequently ignore the governing laws behind them, they can only achieve sub-optimal performance compared to the proposed PhyMo.

\begin{table*}[t]
\centering
\caption{Experimental comparison results across datasets.}
\label{tab:sheet1-r2}
\setlength{\tabcolsep}{1.4pt}
\renewcommand{\arraystretch}{1.05}
\newcommand{\rstd}[2]{%
  \makebox[5.2em][c]{$#1_{\pm #2}$}%
}
\newcommand{\brstd}[2]{%
  \makebox[5.2em][c]{\boldmath$#1_{\pm #2}$}%
}
\scriptsize
\resizebox{\textwidth}{!}{%
\begin{tabular}{@{}c*{11}{c}@{}}
\toprule
Dataset & LASSO & RF & LSTM & TabM & TSMixer & xLSTM & BiMamba & AstroCLIP & Maven & AION-1 & \textbf{PhyMo (ours)} \\
\midrule
SKIPP'D
& \rstd{82.61}{0.01}
& \rstd{83.44}{0.03}
& \rstd{83.65}{0.48}
& \rstd{83.82}{0.26}
& \rstd{83.69}{0.42}
& \rstd{83.65}{0.30}
& \rstd{83.13}{0.56}
& \rstd{82.62}{0.31}
& \rstd{75.86}{1.68}
& \rstd{77.94}{1.55}
& \textbf{\brstd{84.80}{0.14}} \\

Folsom
& \rstd{94.44}{0.03}
& \rstd{95.27}{0.05}
& \rstd{95.17}{0.03}
& \rstd{95.21}{0.02}
& \rstd{95.18}{0.02}
& \rstd{95.17}{0.02}
& \rstd{95.18}{0.03}
& \rstd{94.63}{0.03}
& \rstd{92.19}{0.09}
& \rstd{94.44}{0.03}
& \textbf{\brstd{95.35}{0.03}} \\

NREL
& \rstd{73.40}{0.01}
& \rstd{74.98}{0.01}
& \rstd{77.37}{0.03}
& \rstd{76.23}{0.02}
& \rstd{75.50}{0.04}
& \rstd{77.33}{0.08}
& \rstd{77.33}{0.06}
& \rstd{77.38}{0.03}
& \rstd{72.67}{0.50}
& \rstd{73.72}{0.21}
& \textbf{\brstd{79.31}{0.26}} \\

MeteoNet
& \rstd{96.38}{0.06}
& \rstd{93.38}{0.06}
& \rstd{97.65}{0.06}
& \rstd{97.61}{0.04}
& \rstd{97.64}{0.01}
& \textbf{\rstd{97.78}{0.01}}
& \rstd{97.76}{0.03}
& \rstd{97.45}{0.03}
& \rstd{91.23}{0.58}
& \rstd{93.81}{0.22}
& \brstd{98.36}{0.07} \\

Boreas
& \rstd{96.08}{2.01}
& \rstd{13.11}{5.23}
& \rstd{92.77}{5.98}
& \brstd{96.95}{2.82}
& \rstd{96.48}{1.85}
& \rstd{96.07}{2.23}
& \textbf{\rstd{96.24}{1.25}}
& \rstd{38.81}{11.55}
& \rstd{29.01}{2.71}
& \rstd{60.59}{6.99}
& \rstd{96.84}{1.28} \\
\bottomrule
\end{tabular}
}
\end{table*}

\subsection{Ablation Analysis}
\label{sec: ablation}
In order to answer \textbf{Q2}, we further conduct the ablation analysis to the proposed PhyMo, as shown in Table~\ref{tab:ablation-r2-skippd}. As shown in Table~\ref{tab:ablation-r2-skippd}, jointly incorporating the image and physical-field modalities yields an \(R^2\) of \(84.80\), outperforming both unimodal variants. Removing the physical-field modality decreases performance to \(82.90\), while removing the image modality results in \(83.45\). These consistent degradations demonstrate that the two modalities provide complementary information: images capture the observable appearance, whereas the physical-field modality represents the underlying complex physical dynamics of real world scenarios. The larger degradation caused by removing the physical field further indicates that its physics-grounded representation contributes information that cannot be fully recovered from visual observations alone.
The fusion ablation further shows that early fusion achieves \(84.80\), compared with \(83.94\) for late fusion. This result suggests that representation-level interaction is important for exploiting cross-modal dependencies. Early fusion (i.e., EF in Table~\ref{tab:ablation-r2-skippd}) allows visual patterns to be jointly interpreted with their corresponding physical states before prediction, thereby facilitating a unified representation of observable phenomena and governing dynamics. By contrast, late fusion (i.e., LF in Table~\ref{tab:ablation-r2-skippd}) combines modality-specific predictions only after the relevant representations have been learned independently, limiting fine-grained cross-modal interaction.
\begin{table}[t]
\centering
\caption{Ablation analysis of the proposed model on SKIPP'D dataset.}
\label{tab:ablation-r2-skippd}
\setlength{\tabcolsep}{4pt}
\renewcommand{\arraystretch}{1.18}
\newcommand{\rstd}[2]{%
  \makebox[5.2em][c]{$#1_{\pm #2}$}%
}
\newcommand{\brstd}[2]{%
  \makebox[5.2em][c]{\boldmath$#1_{\pm #2}$}%
}
\footnotesize
\begin{tabular}{@{}cccc@{}}
\toprule
Experiment & Setting & Configuration & \shortstack{$R^2$} \\
\midrule
\multirow{3}{*}{Modality}
& Both & Both modalities & \brstd{84.80}{0.14} \\
& w/o field & Field modality removed & \textbf{\rstd{82.90}{0.21}} \\
& w/o image & Image modality removed & \rstd{83.45}{0.14} \\
\midrule
\multirow{2}{*}{Fusion}
& EF & Early fusion & \textbf{\brstd{84.80}{0.14}} \\
& LF & Late fusion & \rstd{83.94}{0.13} \\
\bottomrule
\end{tabular}

\vspace{0.25em}
\end{table}

\subsection{Hyperparameter Sensitivity Analysis}
\label{sec: hyperparameter}
In order to answer \textbf{Q3}, we also conduct the hyperparameter sensitivity analysis to the proposed model, where the image span means the number of images fed into the model during the training and the reconstruction horizon means the horizons reconstructed by the model in Stage~1, as shown in Table~\ref{tab:hyperparam-sensitivity}. It shows that PhyMo is generally robust to the examined hyperparameters, while favoring moderate configurations. For the cross-modal embedding size, the performance remains within a narrow range across all settings, as well as the image span. For the reconstruction horizon, performance improves as the reconstruction set expands from K1 to K3 and remains comparable under K4 and K6, demonstrating that multi-horizon reconstruction provides richer temporal supervision but exhibits diminishing returns once the principal temporal scales are covered, which may be contributed to the fact that PV forecasting in SKIPP'D is more prone to the short-period histories instead of the long-term ones. Overall, PhyMo is not highly sensitive to small hyperparameter variations, although compact cross-modal representations, recent visual context, and a moderately diverse reconstruction horizon offer the most effective balance.


\begin{table}[t]
\centering
\caption{Sensitivity analysis of the proposed model on the SKIPP'D dataset.}
\label{tab:hyperparam-sensitivity}

\setlength{\tabcolsep}{3.3pt}
\renewcommand{\arraystretch}{1.18}
\footnotesize

\newcommand{\firstcell}[1]{\makebox[4.2em][c]{#1}}
\newcommand{\secondcell}[1]{\makebox[6em][c]{#1}}
\newcommand{\metriccell}[1]{\makebox[5.2em][c]{#1}}
\newcommand{\rstd}[2]{%
\makebox[5.2em][c]{$#1_{\pm #2}$}%
}
\newcommand{\brstd}[2]{%
\makebox[5.2em][c]{\boldmath$#1_{\pm #2}$}%
}

\begin{minipage}[t]{0.492\linewidth}
\vspace{0pt}
\centering
\textbf{(a) Cross-modal embedding size}\par\vspace{2pt}
\resizebox{\linewidth}{!}{%
\begin{tabular}{@{}cccccc@{}}
\toprule
\firstcell{EMB}
& \secondcell{Rep. width}
& \metriccell{H30}
& \metriccell{H60}
& \metriccell{H120}
& \metriccell{Mean} \\
\midrule
\firstcell{64}
& \secondcell{192}
& \rstd{87.23}{0.33}
& \rstd{83.20}{0.43}
& \rstd{77.13}{0.82}
& \rstd{84.71}{0.33} \\

\firstcell{96}
& \secondcell{288}
& \rstd{87.20}{0.17}
& \rstd{83.16}{0.36}
& \rstd{77.35}{0.24}
& \brstd{84.80}{0.14} \\

\firstcell{160}
& \secondcell{480}
& \rstd{86.88}{0.38}
& \rstd{83.05}{0.29}
& \rstd{76.63}{0.47}
& \rstd{84.49}{0.18} \\

\firstcell{192}
& \secondcell{576}
& \rstd{91.51}{0.15}
& \rstd{83.36}{0.55}
& \rstd{76.26}{0.54}
& \rstd{84.67}{0.26} \\

\firstcell{256}
& \secondcell{768}
& \rstd{87.13}{0.19}
& \rstd{83.27}{0.29}
& \rstd{76.66}{0.56}
& \rstd{84.68}{0.28} \\
\bottomrule
\end{tabular}%
}
\end{minipage}%
\hfill
\begin{minipage}[t]{0.492\linewidth}
\vspace{0pt}
\centering
\textbf{(b) Image span}\par\vspace{2pt}
\resizebox{\linewidth}{!}{%
\begin{tabular}{@{}cccccc@{}}
\toprule
\firstcell{Setting}
& \secondcell{Span (frames)}
& \metriccell{H30}
& \metriccell{H60}
& \metriccell{H120}
& \metriccell{Mean} \\
\midrule
\firstcell{C5}
& \secondcell{6}
& \rstd{87.17}{0.23}
& \rstd{83.47}{0.52}
& \rstd{77.14}{0.51}
& \brstd{84.80}{0.14} \\

\firstcell{C15}
& \secondcell{16}
& \rstd{87.26}{0.29}
& \rstd{83.13}{0.52}
& \rstd{76.98}{1.09}
& \rstd{84.72}{0.21} \\

\firstcell{C30}
& \secondcell{31}
& \rstd{87.22}{0.33}
& \rstd{82.98}{0.40}
& \rstd{77.08}{0.97}
& \rstd{84.72}{0.44} \\

\firstcell{C45}
& \secondcell{46}
& \rstd{87.03}{0.43}
& \rstd{82.53}{0.82}
& \rstd{76.16}{0.74}
& \rstd{84.29}{0.40} \\

\firstcell{C60}
& \secondcell{61}
& \rstd{86.85}{0.25}
& \rstd{82.58}{0.26}
& \rstd{76.73}{0.99}
& \rstd{84.34}{0.23} \\
\bottomrule
\end{tabular}%
}
\end{minipage}

\vspace{0.25em}

\begin{minipage}[t]{0.7\linewidth}
\vspace{0pt}
\centering
\textbf{(c) Reconstruction horizon}\par\vspace{2pt}
\resizebox{\linewidth}{!}{%
\begin{tabular}{@{}cccccc@{}}
\toprule
Setting
& Reconstruction set (min)
& H30
& H60
& H120
& Mean \\
\midrule
K1
& $\{5\}$
& \rstd{87.08}{0.39}
& \rstd{83.17}{0.51}
& \rstd{76.50}{0.59}
& \rstd{83.52}{0.37} \\

K2
& $\{5,15\}$
& \rstd{86.95}{0.30}
& \rstd{82.83}{0.84}
& \rstd{76.87}{0.82}
& \rstd{83.85}{0.51} \\

K3
& $\{5,15,30\}$
& \rstd{87.25}{0.25}
& \rstd{83.17}{0.38}
& \rstd{76.70}{0.77}
& \brstd{84.80}{0.14} \\

K4
& $\{5,15,20,30\}$
& \rstd{87.21}{0.19}
& \rstd{83.08}{0.52}
& \rstd{76.26}{0.64}
& \rstd{84.75}{0.26} \\

K6
& $\{5,10,15,20,25,30\}$
& \rstd{86.88}{0.50}
& \rstd{83.03}{0.64}
& \rstd{76.78}{0.70}
& \rstd{84.78}{0.47} \\
\bottomrule
\end{tabular}%
}
\end{minipage}

\vspace{0.35em}
\end{table}

\subsection{Prediction Visualization}
\label{sec: visualization}
In order to answer \textbf{Q4}, we visualize the prediction results of PhyMo on SKIPP'D, as shown in Fig.~\ref{fig:visualization}. The visualization reveals a clear dependence of forecasting behavior on weather-induced variability. Under sunny conditions, PhyMo closely tracks the ground-truth curves throughout the day, accurately recovering both the smooth diurnal evolution and the timing and magnitude of peak PV output, as reflected by consistently low RMSE and MAE values. Under cloudy conditions, the model still captures the overall generation envelope and major ramping events, including changes in the onset, peak, and decay of PV output. However, the errors increase as the ground truth becomes more volatile, particularly in periods characterized by abrupt, short-lived fluctuations.

\begin{figure*}[!h]
    \centering
    \includegraphics[width=\textwidth]{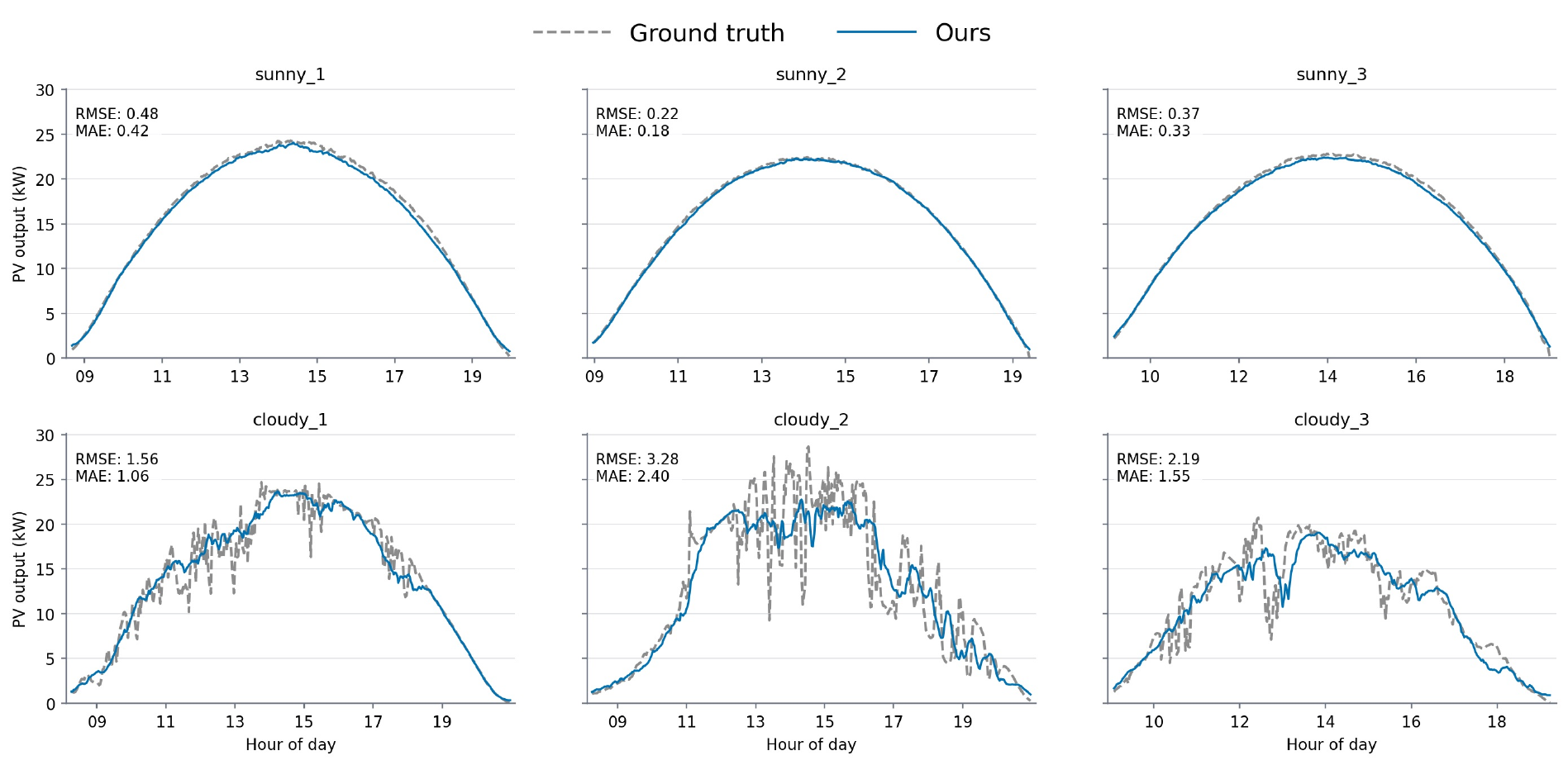}
    \caption{
    The prediction visualization of the proposed PhyMo on SKIPP'D. 
    }
    \label{fig:visualization}
\end{figure*}


\section{Conclusion}
\label{sec: conclusion}
This work introduced the \textbf{physical-field modality} and \textbf{PhyMo}, a physics-grounded multimodal framework for representing physical measurements according to their governing dynamics rather than as generic numerical or textual tokens. PhyMo organizes heterogeneous physical quantities through PDE-associated operators and adopts a three-stage learning procedure comprising physical-field encoder pretraining, cross-modal alignment, and downstream prediction.
Our theoretical analyses demonstrate the benefits of complementary multimodal information, physics-constrained learning, and PDE-structured inputs.
Across five datasets covering diverse physical environments, PhyMo achieves SOTA performance, compared to the strongest baseline on each dataset. These results establish physical fields as a principled modality for multimodal AI4Physics and provide a scalable framework for incorporating additional physical quantities, governing equations, and scientific domains.

\bibliography{references}
\bibliographystyle{iclr2027_conference}

\newpage
\appendix

\section{Detailed Theoretical Analysis}
\label{app:theory}

In this section, we provide detailed proofs for the theoretical results presented in Section~\ref{sec: theo_analysis}. We use squared prediction loss throughout the analysis, which is consistent with the MSE objective used for downstream prediction.

\subsection{Preliminaries}

Let $Y \in \mathbb{R}$ denote the target to be predicted. Let
$\mathbf{Z}_{\mathrm{img}}$ and $\mathbf{Z}_{\mathrm{phy}}$
denote the learned image and physical-field representations, respectively.
For an arbitrary representation $\mathbf{Z}$, define its Bayes-optimal prediction risk under squared loss as
\begin{equation}
\mathcal{R}^{*}(\mathbf{Z})
=
\inf_{f}
\mathbb{E}
\left[
\left(
Y-f(\mathbf{Z})
\right)^2
\right],
\label{eq:bayes_risk}
\end{equation}
where the infimum is taken over all measurable predictors $f$.

Under squared loss, the Bayes-optimal predictor is
\begin{equation}
f^{*}(\mathbf{Z})
=
\mathbb{E}[Y\mid\mathbf{Z}],
\label{eq:bayes_predictor}
\end{equation}
and therefore
\begin{equation}
\mathcal{R}^{*}(\mathbf{Z})
=
\mathbb{E}
\left[
\operatorname{Var}
\left(
Y\mid\mathbf{Z}
\right)
\right].
\label{eq:bayes_condvar}
\end{equation}

This identity provides the basis for Theorems~\ref{thm:multimodal} and~\ref{thm:pde_structure}.

\subsection{Proof of Multimodal Bayes-Risk Dominance}
\label{app:multimodal}

\begin{proof}[Proof of Theorem~\ref{thm:multimodal}]

Consider first the image representation $\mathbf{Z}_{\mathrm{img}}$. By the law of total variance,
\begin{align}
\operatorname{Var}
\left(
Y\mid\mathbf{Z}_{\mathrm{img}}
\right)
&=
\mathbb{E}
\left[
\operatorname{Var}
\left(
Y
\mid
\mathbf{Z}_{\mathrm{img}},
\mathbf{Z}_{\mathrm{phy}}
\right)
\middle|
\mathbf{Z}_{\mathrm{img}}
\right]
\nonumber\\
&\quad+
\operatorname{Var}
\left(
\mathbb{E}
\left[
Y
\mid
\mathbf{Z}_{\mathrm{img}},
\mathbf{Z}_{\mathrm{phy}}
\right]
\middle|
\mathbf{Z}_{\mathrm{img}}
\right).
\label{eq:variance_decomp_img}
\end{align}

The second term on the right-hand side is non-negative. Hence,
\begin{equation}
\operatorname{Var}
\left(
Y\mid\mathbf{Z}_{\mathrm{img}}
\right)
\geq
\mathbb{E}
\left[
\operatorname{Var}
\left(
Y
\mid
\mathbf{Z}_{\mathrm{img}},
\mathbf{Z}_{\mathrm{phy}}
\right)
\middle|
\mathbf{Z}_{\mathrm{img}}
\right].
\end{equation}

Taking expectation over $\mathbf{Z}_{\mathrm{img}}$ gives
\begin{equation}
\mathbb{E}
\left[
\operatorname{Var}
\left(
Y\mid\mathbf{Z}_{\mathrm{img}}
\right)
\right]
\geq
\mathbb{E}
\left[
\operatorname{Var}
\left(
Y
\mid
\mathbf{Z}_{\mathrm{img}},
\mathbf{Z}_{\mathrm{phy}}
\right)
\right].
\end{equation}

Using Eq.~\eqref{eq:bayes_condvar}, we obtain
\begin{equation}
\mathcal{R}^{*}
\left(
\mathbf{Z}_{\mathrm{img}},
\mathbf{Z}_{\mathrm{phy}}
\right)
\leq
\mathcal{R}^{*}
\left(
\mathbf{Z}_{\mathrm{img}}
\right).
\label{eq:risk_img}
\end{equation}

By symmetry, conditioning first on $\mathbf{Z}_{\mathrm{phy}}$ yields
\begin{equation}
\mathcal{R}^{*}
\left(
\mathbf{Z}_{\mathrm{img}},
\mathbf{Z}_{\mathrm{phy}}
\right)
\leq
\mathcal{R}^{*}
\left(
\mathbf{Z}_{\mathrm{phy}}
\right).
\label{eq:risk_phy}
\end{equation}

Combining Eqs.~\eqref{eq:risk_img} and~\eqref{eq:risk_phy} gives
\begin{equation}
\boxed{
\mathcal{R}^{*}
\left(
\mathbf{Z}_{\mathrm{img}},
\mathbf{Z}_{\mathrm{phy}}
\right)
\leq
\min
\left\{
\mathcal{R}^{*}
\left(
\mathbf{Z}_{\mathrm{img}}
\right),
\mathcal{R}^{*}
\left(
\mathbf{Z}_{\mathrm{phy}}
\right)
\right\}.
}
\end{equation}
\end{proof}

\subsection{Proof of Reconstruction-Risk Advantage of Physics-Constrained Representations}
\label{app:physics_reconstruction}

We provide the detailed proof of
Theorem~\ref{thm:physics_reconstruction} and establish its connection to the Stage~1 training objective.

\subsubsection{Physical Admissible Set}

Let
\begin{equation}
\mathbf{X}_{\mathrm{phy}}^{*}
=
\left\{
\mathbf{X}_{\mathrm{phy}}^{*(1)},
\ldots,
\mathbf{X}_{\mathrm{phy}}^{*(M)}
\right\}
\end{equation}
denote the underlying physical fields.
For the $m$-th field, let $\mathcal{F}_m$ denote its governing PDE.

We define the joint physically admissible set as
\begin{equation}
\mathcal{S}
=
\left\{
\mathbf{X}:
\mathcal{F}_m
\left(
\mathbf{X}^{(m)}
\right)=0,
\quad
m=1,\ldots,M
\right\}.
\label{eq:admissible_set_app}
\end{equation}

We assume that the governing equations are correctly specified, such that
\begin{equation}
\mathbf{X}_{\mathrm{phy}}^{*}
\in
\mathcal{S}.
\label{eq:true_in_set}
\end{equation}

We further assume that $\mathcal{S}$ is nonempty, closed, and convex, so that the Euclidean projection onto $\mathcal{S}$ is uniquely defined.

Let
\begin{equation}
\tilde{\mathbf{X}}_{\mathrm{phy}}
=
D_{\mathrm{phy}}(\mathbf{Z})
\end{equation}
denote an arbitrary unconstrained reconstruction obtained from a latent representation $\mathbf{Z}$.
The corresponding physics-constrained reconstruction is defined by
\begin{equation}
\hat{\mathbf{X}}_{\mathrm{phy}}^{*}
=
\Pi_{\mathcal{S}}
\left(
\tilde{\mathbf{X}}_{\mathrm{phy}}
\right)
=
\arg\min_{\mathbf{U}\in\mathcal{S}}
\left\|
\mathbf{U}
-
\tilde{\mathbf{X}}_{\mathrm{phy}}
\right\|_2^2.
\label{eq:projection_def}
\end{equation}

\subsubsection{Projection-Based Reconstruction Bound}

\begin{proof}[Proof of Theorem~\ref{thm:physics_reconstruction}]

For any nonempty closed convex set $\mathcal{S}$, the metric projection
$\Pi_{\mathcal{S}}(\mathbf{x})$
satisfies
\begin{equation}
\left\langle
\mathbf{x}
-
\Pi_{\mathcal{S}}(\mathbf{x}),
\mathbf{s}
-
\Pi_{\mathcal{S}}(\mathbf{x})
\right\rangle
\leq 0,
\qquad
\forall \mathbf{s}\in\mathcal{S}.
\label{eq:projection_property}
\end{equation}

Taking
\begin{equation}
\mathbf{x}
=
\tilde{\mathbf{X}}_{\mathrm{phy}},
\qquad
\mathbf{s}
=
\mathbf{X}_{\mathrm{phy}}^{*},
\end{equation}
and using
$\mathbf{X}_{\mathrm{phy}}^{*}\in\mathcal{S}$,
we obtain
\begin{equation}
\left\langle
\tilde{\mathbf{X}}_{\mathrm{phy}}
-
\hat{\mathbf{X}}_{\mathrm{phy}}^{*},
\mathbf{X}_{\mathrm{phy}}^{*}
-
\hat{\mathbf{X}}_{\mathrm{phy}}^{*}
\right\rangle
\leq 0.
\label{eq:projection_true}
\end{equation}

Now decompose the unconstrained reconstruction error as
\begin{equation}
\tilde{\mathbf{X}}_{\mathrm{phy}}
-
\mathbf{X}_{\mathrm{phy}}^{*}
=
\left(
\tilde{\mathbf{X}}_{\mathrm{phy}}
-
\hat{\mathbf{X}}_{\mathrm{phy}}^{*}
\right)
+
\left(
\hat{\mathbf{X}}_{\mathrm{phy}}^{*}
-
\mathbf{X}_{\mathrm{phy}}^{*}
\right).
\end{equation}

Taking the squared Euclidean norm gives
\begin{align}
&
\left\|
\tilde{\mathbf{X}}_{\mathrm{phy}}
-
\mathbf{X}_{\mathrm{phy}}^{*}
\right\|_2^2
\nonumber\\
&=
\left\|
\tilde{\mathbf{X}}_{\mathrm{phy}}
-
\hat{\mathbf{X}}_{\mathrm{phy}}^{*}
\right\|_2^2
+
\left\|
\hat{\mathbf{X}}_{\mathrm{phy}}^{*}
-
\mathbf{X}_{\mathrm{phy}}^{*}
\right\|_2^2
\nonumber\\
&\quad+
2
\left\langle
\tilde{\mathbf{X}}_{\mathrm{phy}}
-
\hat{\mathbf{X}}_{\mathrm{phy}}^{*},
\hat{\mathbf{X}}_{\mathrm{phy}}^{*}
-
\mathbf{X}_{\mathrm{phy}}^{*}
\right\rangle.
\label{eq:error_expansion}
\end{align}

From Eq.~\eqref{eq:projection_true},
\begin{equation}
\left\langle
\tilde{\mathbf{X}}_{\mathrm{phy}}
-
\hat{\mathbf{X}}_{\mathrm{phy}}^{*},
\hat{\mathbf{X}}_{\mathrm{phy}}^{*}
-
\mathbf{X}_{\mathrm{phy}}^{*}
\right\rangle
\geq 0.
\end{equation}

Therefore,
\begin{align}
\left\|
\tilde{\mathbf{X}}_{\mathrm{phy}}
-
\mathbf{X}_{\mathrm{phy}}^{*}
\right\|_2^2
&\geq
\left\|
\tilde{\mathbf{X}}_{\mathrm{phy}}
-
\hat{\mathbf{X}}_{\mathrm{phy}}^{*}
\right\|_2^2
\nonumber\\
&\quad+
\left\|
\hat{\mathbf{X}}_{\mathrm{phy}}^{*}
-
\mathbf{X}_{\mathrm{phy}}^{*}
\right\|_2^2.
\label{eq:pythagorean}
\end{align}

Since the first term on the right-hand side is non-negative,
\begin{equation}
\boxed{
\left\|
\hat{\mathbf{X}}_{\mathrm{phy}}^{*}
-
\mathbf{X}_{\mathrm{phy}}^{*}
\right\|_2^2
\leq
\left\|
\tilde{\mathbf{X}}_{\mathrm{phy}}
-
\mathbf{X}_{\mathrm{phy}}^{*}
\right\|_2^2.
}
\label{eq:projection_final}
\end{equation}

Let the reconstruction loss associated with a representation
$\mathbf{Z}$ be
\begin{equation}
\mathcal{L}_{\mathrm{rec}}(\mathbf{Z})
=
\left\|
D_{\mathrm{phy}}(\mathbf{Z})
-
\mathbf{X}_{\mathrm{phy}}^{*}
\right\|_2^2.
\label{eq:rec_loss_true}
\end{equation}

If $\mathbf{Z}_{\mathrm{phy}}^{*}$ denotes a representation whose decoded field equals
$\hat{\mathbf{X}}_{\mathrm{phy}}^{*}$,
Eq.~\eqref{eq:projection_final} directly implies
\begin{equation}
\boxed{
\mathcal{L}_{\mathrm{rec}}
\left(
\mathbf{Z}_{\mathrm{phy}}^{*}
\right)
\leq
\mathcal{L}_{\mathrm{rec}}
\left(
\mathbf{Z}
\right).
}
\label{eq:latent_rec_dominance}
\end{equation}

Since Eq.~\eqref{eq:latent_rec_dominance} holds pointwise, taking expectation over the data distribution preserves the inequality:
\begin{equation}
\boxed{
\mathbb{E}
\left[
\mathcal{L}_{\mathrm{rec}}
\left(
\mathbf{Z}_{\mathrm{phy}}^{*}
\right)
\right]
\leq
\mathbb{E}
\left[
\mathcal{L}_{\mathrm{rec}}
\left(
\mathbf{Z}
\right)
\right].
}
\label{eq:expected_rec_final}
\end{equation}
\end{proof}

\subsection{Proof of the Advantage of PDE-Structured Physical Representations}
\label{app:pde_structure}

We finally compare the proposed physical-field organization with a conventional representation that treats physical measurements only as numerical values.

Let
\begin{equation}
\mathbf{X}
=
\left[
R,u,v,T,q,\ldots
\right]
\end{equation}
denote the physical quantities observed at a given state.

Let
\begin{equation}
\mathcal{D}(\mathbf{X})
=
\left[
\partial_t \mathbf{X},
\nabla \mathbf{X},
\nabla^2\mathbf{X},
\partial_{xy}\mathbf{X},
\ldots
\right]
\end{equation}
denote the collection of the PDE operators used by the proposed physical-field modality.
A conventional numerical representation uses only $\mathbf{X}$, whereas the PDE-structured representation uses
\begin{equation}
\mathbf{X}_{\mathrm{PDE}}
=
\left(
\mathbf{X},
\mathcal{D}(\mathbf{X})
\right).
\end{equation}

\begin{proof}[Proof of Theorem~\ref{thm:pde_structure}]

Using Eq.~\eqref{eq:bayes_condvar}, the minimum prediction risk based only on numerical physical values is
\begin{equation}
\mathcal{R}^{*}(\mathbf{X})
=
\mathbb{E}
\left[
\operatorname{Var}
(Y\mid\mathbf{X})
\right].
\end{equation}

The Bayes risk using the PDE-structured physical information is
\begin{equation}
\mathcal{R}^{*}
\left(
\mathbf{X},
\mathcal{D}(\mathbf{X})
\right)
=
\mathbb{E}
\left[
\operatorname{Var}
\left(
Y
\mid
\mathbf{X},
\mathcal{D}(\mathbf{X})
\right)
\right].
\end{equation}

Applying the law of total variance conditional on $\mathbf{X}$ gives
\begin{align}
\operatorname{Var}(Y\mid\mathbf{X})
&=
\mathbb{E}
\left[
\operatorname{Var}
\left(
Y
\mid
\mathbf{X},
\mathcal{D}(\mathbf{X})
\right)
\middle|
\mathbf{X}
\right]
\nonumber\\
&\quad+
\operatorname{Var}
\left(
\mathbb{E}
\left[
Y
\mid
\mathbf{X},
\mathcal{D}(\mathbf{X})
\right]
\middle|
\mathbf{X}
\right).
\label{eq:pde_var_decomp}
\end{align}

Since the second term is non-negative,
we have,
\begin{equation}
\boxed{
\mathcal{R}^{*}
\left(
\mathbf{X},
\mathcal{D}(\mathbf{X})
\right)
\leq
\mathcal{R}^{*}(\mathbf{X}).
}
\end{equation}

Moreover, subtracting the two risks gives
\begin{align}
&
\mathcal{R}^{*}(\mathbf{X})
-
\mathcal{R}^{*}
\left(
\mathbf{X},
\mathcal{D}(\mathbf{X})
\right)
\nonumber\\
&=
\mathbb{E}
\left[
\operatorname{Var}
\left(
\mathbb{E}
\left[
Y
\mid
\mathbf{X},
\mathcal{D}(\mathbf{X})
\right]
\middle|
\mathbf{X}
\right)
\right].
\label{eq:pde_gain}
\end{align}

Hence, the inequality is strict whenever
\begin{equation}
\operatorname{Var}
\left(
\mathbb{E}
[
Y
\mid
\mathbf{X},
\mathcal{D}(\mathbf{X})
]
\mid
\mathbf{X}
\right)
>0
\end{equation}
with positive probability.

This condition precisely means that the differential structure contains predictive information that cannot be recovered from the instantaneous numerical state alone.
\end{proof}

\section{The Detailed Illustration about the Experiment Setup}
\label{appen: pre-processing}

\subsection{Dataset Details}
\label{appen: datasets}

We conduct experiments on five real-world datasets covering different geographic locations, atmospheric conditions, and forecasting targets. These datasets provide complementary combinations of sky imagery, irradiance measurements, photovoltaic generation, and meteorological observations, allowing us to evaluate the proposed multimodal framework under diverse physical environments.

\paragraph{SKIPP'D.}
The SKIPP'D dataset~(\cite{nie2023skipp}) was collected at Stanford University and was designed specifically for short-term solar forecasting from ground-based sky imagery. It contains three years of observations from 2017 to 2019 and provides temporally synchronized sky images and photovoltaic (PV) power generation measurements. The benchmark version includes quality-controlled sky images downsampled to $64\times64$ pixels and concurrent PV generation records at a one-minute interval. In addition, the dataset releases the corresponding raw data, including high-resolution $2048\times2048$ sky images, high-frequency sky video recorded at 20 frames per second, and PV generation histories at one-minute resolution. This combination provides both rich visual observations of cloud evolution and direct measurements of the downstream energy-generation process, making SKIPP'D particularly suitable for evaluating multimodal PV power forecasting models. In our setting, historical sky images constitute the visual modality, while the synchronized physical measurements are used to construct the physical-field representation. The forecasting target is future PV power generation. 

\paragraph{Folsom.}
The Folsom dataset~(\cite{pedro2019comprehensive}) is a comprehensive solar forecasting benchmark collected in California over the period 2014--2016. It provides quality-controlled irradiance observations at one-minute resolution, including global horizontal irradiance (GHI), direct normal irradiance (DNI), and diffuse horizontal irradiance (DHI). The dataset additionally contains daytime ground-based sky images sampled at one-minute intervals, local meteorological observations, GOES-15 satellite information, and Numerical Weather Prediction (NWP) forecasts from the North American Mesoscale model. All temporal records are provided in UTC, while irradiance and meteorological quantities are stored in SI units. Compared with SKIPP'D, which directly provides PV generation as the prediction target, Folsom primarily supports solar irradiance forecasting and therefore offers a complementary evaluation setting in which the target is an atmospheric radiative quantity rather than electrical power output. Its synchronized imagery and heterogeneous meteorological observations make it particularly useful for evaluating whether physical-field representations improve forecasting beyond purely image-based models.

\paragraph{NREL Sky Imagery.}
The NREL Sky Imagery dataset~(\cite{hammond2026nrel}) is collected at the Solar Radiation Research Laboratory (SRRL) of the National Renewable Energy Laboratory in Golden, Colorado, USA. The dataset uses an EKO ASI-16 all-sky imager operating at its native one-minute acquisition frequency and provides high-resolution $1920\times1920$ JPEG sky images. Data collection began in October 2025 and the dataset is continuously updated, providing substantially denser temporal sampling than the standard ten-minute SRRL public image gallery. Each daytime image is temporally associated with the nearest one-minute observation from the SRRL Baseline Measurement System, with a matching tolerance of approximately five minutes.

The accompanying physical measurements provide a particularly rich set of atmospheric variables. They include GHI, DNI, DHI, secondary irradiance measurements, ambient air temperature, relative humidity, average and peak wind speed, wind direction, atmospheric pressure, precipitation, solar zenith angle, solar azimuth angle, and a derived clear-sky index. These variables jointly describe radiative, thermodynamic, and dynamical aspects of the local atmosphere and therefore provide a natural basis for constructing the physical-field modality. The dataset is intended primarily for intra-hour solar irradiance forecasting, cloud-motion analysis, and cross-site transfer learning. In our experiments, historical all-sky images provide the visual stream, while synchronized irradiance and meteorological quantities constitute the physical observations used by the physical-field encoder.

\paragraph{MeteoNet.}
The MeteoNet dataset~(\cite{larvor2020meteonet}) is an open meteorological dataset released by Météo-France. It covers the period from 2016 to 2018 over two $550\,\mathrm{km}\times550\,\mathrm{km}$ regions in northwestern and southeastern France. The dataset integrates multiple heterogeneous data sources, including ground-station observations, precipitation radar maps, satellite imagery, numerical weather prediction outputs, and static land--sea and terrain masks. Ground observations are sampled every six minutes, radar products every five minutes at a spatial resolution of $0.01^\circ$, and satellite cloud-type maps every 15 minutes at a resolution of $0.03^\circ$.

The physical measurements include temperature, dew-point temperature, humidity, atmospheric pressure, precipitation, wind speed, and wind direction. MeteoNet further provides two- and three-dimensional forecasts from the AROME and ARPEGE models, containing near-surface and upper-atmospheric variables such as temperature, humidity, horizontal wind components, vertical velocity, pressure, and geopotential. These temporally aligned observations, remote-sensing products, and forecast fields make MeteoNet suitable for multimodal weather forecasting, precipitation nowcasting, and learning representations that jointly capture radiative, thermodynamic, moisture, and dynamical atmospheric processes.

\paragraph{Boreas.}
The Boreas dataset~(\cite{burnett2023boreas}) is a multi-season autonomous-driving dataset collected along repeated routes near Toronto, Canada. It contains 44 driving sequences covering more than 350 km between November 2020 and November 2021, encompassing substantial seasonal and weather variations, including sunshine, rain, accumulated snow, and active snowfall. Its sensing platform combines a 128-beam Velodyne Alpha-Prime lidar, a $2448\times2048$ FLIR Blackfly S monocular camera, a $360^\circ$ Navtech scanning radar, and an Applanix POS LV GNSS--INS system. The lidar and camera operate at 10 Hz, while the radar operates at 4 Hz.

The dataset provides synchronized sensor streams, intrinsic and extrinsic calibration parameters, and centimetre-level post-processed ground-truth poses. Repeated observations of the same environment under different seasons and weather conditions enable systematic evaluation of multimodal representation learning and sensor robustness under distribution shifts. Boreas primarily supports all-weather odometry, long-term metric localization, sensor fusion, and 3D object detection. Its complementary camera, lidar, radar, and navigation measurements are particularly suitable for studying how different sensing modalities respond to adverse environmental conditions.



\subsection{Baseline Details}
\label{appen: baselines}
\paragraph{LASSO.}
LASSO~(\cite{ranstam2018lasso}) addresses overfitting and feature redundancy in regression by augmenting the least-squares objective with an $\ell_1$ penalty. It directly maps the input variables to the target while shrinking weakly informative coefficients toward zero, thereby performing regression and implicit feature selection simultaneously. We include LASSO as a classical structured-data baseline because it provides an interpretable linear reference for assessing whether the gains of more complex models arise from nonlinear temporal and multimodal representation learning.

\paragraph{Random Forest.}
Random Forest (RF)~(\cite{pavlov2000random}) models nonlinear relationships through an ensemble of decision trees. Each tree is trained on a bootstrapped subset of the training data, considers a randomly sampled subset of features at each split, and produces an independent prediction; the final regression output is obtained by averaging the predictions of all trees. RF is selected as a classical structured-data baseline because it captures nonlinear feature interactions without learned temporal representations, providing a strong non-neural reference for evaluating the benefits of sequence modeling and multimodal fusion.

\paragraph{TabM.}
TabM~(\cite{gorishniy2025tabm}) addresses supervised learning on tabular data through a parameter-efficient ensemble of MLPs. It uses BatchEnsemble-style parameter sharing to generate multiple predictions for each sample within a single model and aggregates these predictions to improve optimization, robustness, and generalization. We select TabM as a representative neural baseline for structured inputs because it offers strong regression performance without explicitly modeling temporal order, enabling us to distinguish the contribution of temporal and multimodal inductive biases from that of a powerful numerical feature learner.

\paragraph{LSTM.}
Long Short-Term Memory (LSTM)~(\cite{hochreiter1997long}) was introduced to mitigate the vanishing-gradient problem that prevents conventional recurrent networks from learning long-range dependencies. At each time step, its input, forget, and output gates regulate information written to, retained in, and read from a persistent memory cell; the resulting hidden representation is subsequently mapped to the regression target. As regression tasks are inherently sequential, we include LSTM as the canonical recurrent baseline for evaluating whether more recent architectures provide advantages beyond conventional gated temporal modeling.

\paragraph{TSMixer.}
TSMixer~(\cite{chen2023tsmixer}) is an all-MLP architecture designed for multivariate time-series forecasting. Its stacked mixer blocks alternately apply MLPs along the temporal dimension to capture dependencies across time steps and along the feature dimension to model interactions among variables, after which a projection head produces the forecasts. We include TSMixer as a lightweight temporal baseline because it directly targets regression forecasting while avoiding recurrence and self-attention, allowing us to evaluate whether simple temporal and cross-variable mixing is sufficient for the considered tasks.

\paragraph{xLSTM.}
xLSTM~(\cite{beck2024xlstm}) modernizes the original LSTM through two complementary memory mechanisms: sLSTM introduces exponential gating and enhanced memory mixing, whereas mLSTM replaces scalar memory with a matrix-valued memory and covariance-style updates. These blocks are stacked to encode the input sequence before a task-specific head generates the prediction. We select xLSTM as a recent recurrent baseline because it tests whether improved long-range memory and scalable sequence modeling can better capture the temporal dynamics underlying regression-based prediction.

\paragraph{BiMamba.}
Sky-BiMamba~(\cite{zhang2025multimodal}) addresses ultra-short-term PV power forecasting by jointly modeling historical sky images and PV measurements. It first extracts cloud appearance and motion cues from sky-image sequences using frame differencing and lightweight visual encoders, then employs bidirectional Mamba modules and temporally aware cross-modal attention to integrate visual features with PV sequences; a dynamic gating mechanism further reduces the influence of low-quality images before prediction. We include it as a recent domain-specific multimodal baseline because its task and input modalities closely match our setting, making it a particularly direct comparison for evaluating PhyMo.

\paragraph{AstroCLIP.}
AstroCLIP~(\cite{parker2024astroclip}) addresses the joint representation of galaxy images and optical spectra, which provide complementary observations of the same physical system. It first pretrains transformer-based image and spectrum encoders using modality-specific self-supervised objectives and subsequently aligns paired representations in a shared latent space through contrastive learning. AstroCLIP is selected as a recent scientific multimodal baseline because its separate encoding followed by cross-modal alignment closely matches the general paradigm considered in our work, while providing a reference for comparison with PhyMo's explicitly physics-aware representation.

\paragraph{Maven.}
Maven~(\cite{zhang2024maven}) addresses the substantial imbalance between abundant photometric observations and scarce but information-rich spectra in supernova science. It encodes photometric and spectroscopic observations separately, contrastively aligns their representations using approximately 0.5 million simulated supernovae, and then adapts the pretrained model to real observations for downstream classification and redshift regression. We include Maven as a recent multimodal scientific model because it demonstrates how heterogeneous observations can be unified through contrastive pretraining and transferred to regression tasks, providing a relevant comparison with our multimodal alignment strategy.

\paragraph{AION-1.}
AION-1~(\cite{parker2026aion}) addresses the fragmentation of astronomical learning systems across heterogeneous instruments, surveys, and data types. It first transforms images, spectra, and scalar measurements using modality-specific tokenizers and then applies transformer-based masked modeling to their combined token sequences, producing a unified encoder transferable to property estimation, retrieval, segmentation, and spectral reconstruction. We select AION-1 as a recent omnimodal scientific foundation model because it represents a general approach to integrating heterogeneous scientific observations and therefore provides a strong comparison for assessing the effectiveness of PhyMo's dedicated physical-field organization and learning objective.

\subsection{Metrics}
\label{appen: metrics}
We evaluate model performance using root mean squared error (RMSE), coefficient of determination ($R^2$), and mean absolute error (MAE). Let
$\{y_i\}_{i=1}^{N}$ denote the ground-truth targets,
$\{\hat{y}_i\}_{i=1}^{N}$ denote the corresponding predictions,
and
\begin{equation}
    \bar{y}
    =
    \frac{1}{N}
    \sum_{i=1}^{N} y_i
\end{equation}
denote the mean of the ground-truth targets.

\paragraph{Root Mean Squared Error (RMSE).}
RMSE measures the square root of the average squared prediction error:
\begin{equation}
    \mathrm{RMSE}
    =
    \sqrt{
        \frac{1}{N}
        \sum_{i=1}^{N}
        \left(
            \hat{y}_i-y_i
        \right)^2
    }.
\end{equation}
Because prediction errors are squared before aggregation, RMSE places relatively larger penalties on samples with large deviations. It is therefore particularly sensitive to large forecasting errors. A lower RMSE indicates better predictive accuracy, with an ideal value of zero.

\paragraph{Mean Absolute Error (MAE).}
MAE computes the average absolute difference between predictions and ground-truth values:
\begin{equation}
    \mathrm{MAE}
    =
    \frac{1}{N}
    \sum_{i=1}^{N}
    \left|
        \hat{y}_i-y_i
    \right|.
\end{equation}
In contrast to RMSE, MAE penalizes prediction errors linearly and is therefore less sensitive to a small number of large deviations. It provides a direct measure of the typical absolute forecasting error in the same unit as the prediction target. Lower MAE values indicate better performance.

\paragraph{Coefficient of Determination ($R^2$).}
The coefficient of determination evaluates how much of the variation in the target variable is explained by the model:
\begin{equation}
    R^2
    =
    1
    -
    \frac{
        \sum_{i=1}^{N}
        \left(
            y_i-\hat{y}_i
        \right)^2
    }{
        \sum_{i=1}^{N}
        \left(
            y_i-\bar{y}
        \right)^2
    }.
\end{equation}
The numerator corresponds to the residual sum of squares, while the denominator represents the total variation of the ground-truth targets around their mean. An $R^2$ value of $1$ indicates perfect predictions, whereas $R^2=0$ indicates performance equivalent to predicting the target mean for every sample. Negative $R^2$ values are also possible and indicate that the model performs worse than this mean-value baseline. Unlike RMSE and MAE, larger $R^2$ values correspond to better predictive performance.

\paragraph{Complementarity of the Metrics.}
The three metrics characterize prediction quality from complementary perspectives. RMSE emphasizes large deviations, MAE reflects the average magnitude of forecasting errors with greater robustness to extreme errors, and $R^2$ measures the explanatory capability of the model relative to the intrinsic variability of the target. We therefore report all three metrics to provide a more comprehensive evaluation of downstream forecasting performance.

\section{Limitations}

Despite its effectiveness, PhyMo has several limitations that motivate future work. First, owing to space constraints, our downstream evaluations primarily focus on regression-based forecasting; we will extend PhyMo to broader tasks, including classification, reconstruction, and physical-state estimation, to further assess its applicability and generalizability in real-world physical systems. Second, although our experiments cover diverse datasets and physical environments, the available physical variables and their spatiotemporal resolutions remain dataset-dependent. We will therefore evaluate PhyMo under incomplete, noisy, and asynchronously sampled observations and extend it to additional physical domains. Finally, the three-stage training procedure and PDE residual computation introduce additional computational overhead compared with conventional data-driven models. Future work will improve efficiency through lightweight physical-field encoders, shared-stage optimization, and more efficient PDE-operator computation.

\section*{AI Use Statement}
In this work, we used generative AI tool to assist with translation, implement methods with codes and provide feedback on experiments.
We have not used generative AI tools for generate synthetic data sets, help develop conceptual frameworks, formulate mathematical claims, provide critical ingredients for proving mathematical claims, assist in the writing of proofs, propose hypotheses, clean the dataset, support qualitative and thematic data analysis and interpret results. We have reviewed all AI-assisted work. For example, LLM-generated code was verified and tested for correctness
by 2 authors. We take responsibility for the final content of this work, including text, claims or artifacts produced with the aid of generative AI.


\end{document}

%% file: math_commands.tex
\usepackage{amsmath,amsfonts,bm}

\def\eqref#1{equation~\ref{#1}}

\def\1{\bm{1}}

\DeclareMathAlphabet{\mathsfit}{\encodingdefault}{\sfdefault}{m}{sl}
\SetMathAlphabet{\mathsfit}{bold}{\encodingdefault}{\sfdefault}{bx}{n}

